\documentclass[12pt,reqno]{amsart}
\usepackage{soul}
\usepackage{multicol,amsmath,graphics, mathtools, cancel,soul,color,bbm,amsfonts,dsfont,tikz,mathrsfs,amssymb,bm,bbold,comment}
\usepackage{algorithm}
\usepackage{algorithmicx}
\usepackage{algpseudocode}

\usepackage[foot]{amsaddr}

\usepackage{xcolor}
\usepackage[labelformat=parens,labelsep=space]{subcaption}

\usepackage{graphicx}

\usepackage{stmaryrd}
\usepackage[left=1in, right=1in, top=1.1in,bottom=1.1in]{geometry}
\usetikzlibrary{matrix,patterns,positioning}
\numberwithin{equation}{section}

\usepackage[normalem]{ulem} 
\newcommand{\mathsout}[1]{\ifmmode\text{\sout{\ensuremath{#1}}}\else\sout{#1}\fi}

\usepackage{comment}
\usepackage{soul}

\usepackage[hidelinks]{hyperref}

\newcommand{\R}{\mathbb{R}}

\newcommand{\bbNabla}{%
  \nabla\mkern-11mu\nabla%
}

\newcommand{\vertiii}[1]{{\left\vert\kern-0.25ex\left\vert\kern-0.25ex\left\vert #1
    \right\vert\kern-0.25ex\right\vert\kern-0.25ex\right\vert}}

\def\RR{\mathbb{R}}\def\R{\mathbb{R}}

\def\rd{\mathrm{d}}  

\def\dh2l{\mathbf{d}_{\mathbb{H}_{2\ell}}}
\def\d2{\mathbf{d}_2}

\newtheorem{theorem}{Theorem}[section]

\theoremstyle{remark}

\theoremstyle{definition}

 \numberwithin{dummy}{section}

\def\1{\mathbbm{1}}

\makeatletter
\@namedef{subjclassname@1991}{2020 Mathematics Subject Classification}
\makeatother

\begin{document}

\title{Branching Stein Variational Gradient Descent for sampling multimodal distributions}

\author[Isa\'ias Ba\~nales]{Isa\'ias Ba\~nales$^\dag$}
\address{$\dag$) Kyoto University, Disaster Prevention Research Institute, Japan}
\email{$\dag$) banales.isaias.7v@kyoto-u.ac.jp}

\author[Arturo Jaramillo]{Arturo Jaramillo$^\ast$}
\address{$\ast$) Centro de Investigaci\'on en Matem\'aticas A.C., Department of Probability and Statistics, M\'exico} 
\email{$\ast$)jagil@cimat.mx}

\author[Joshu\'e Hel\'i Ricalde-Guerrero]{Joshu\'e Hel\'i Ricalde-Guerrero$^\ddag$}
\address{$\ddag$) ETH Zurich, Department of Mathematics, Switzerland}
\email{$\ddag$) joshue.ricalde@math.ethz.ch} 

\thanks{All authors contributed equally to the work, the order of appearance is alphabetical. }

\begin{abstract}
We propose a novel particle-based variational inference method designed to work with multimodal distributions. Our approach, referred to as Branched Stein Variational Gradient Descent (BSVGD), extends the classical Stein Variational Gradient Descent (SVGD) algorithm by incorporating a random branching mechanism that encourages the exploration of the state space. In this work, a theoretical guarantee for the convergence in distribution is presented, as well as numerical experiments to validate the suitability of our algorithm. Performance comparisons between the BSVGD and the SVGD are presented using the Wasserstein distance between samples and the corresponding computational times. \\

\end{abstract}

\maketitle




\section{Introduction}

\noindent This paper proposes a variational inference algorithm based on particle methods to sample from multimodal densities. More precisely, we consider the problem of approximating measure of interest \(\pi\), which will be assumed to take the form
\[
\pi(dx) = \frac{1}{Z} \rho(x) dx,
\]
where \(\rho\) is a non-negative, integrable function, and \(Z\) is a normalizing constant. Due to numerical constraints, the normalizing constant Z is often inaccessible, and we work under the standard setting in which only the score function  $\nabla \log \rho$ is available to the user. This will be the standing assumption throughout the paper.\\

\noindent Our starting point is the celebrated Stein Variational Gradient Descent, introduced by Liu and Wang in \cite{liu2016svgd} which is summarized in Algorithm \ref{alg:SVGD}, and further interpreted in the framework of gradient flows in \cite{liu2017svgdflow}, which is the perspective adopted in this paper. Roughly speaking, the method consists of constructing a Wasserstein gradient flow uniquely determined by the score function, and whose asymptotic limit is heuristically close to $\pi$. Convergence guarantees for this method, in the specific setting where the initial condition for the gradient flow is empirical, have been addressed in \cite{shi2024finite}, \cite{liu2024understanding}, and \cite{balasubramanian2024improved}, although the particular frameworks in which these guarantees are phrased is a subtle topic that the reader should take into consideration. The practical implementation of the method has been quite successfull with applications varying between reinforcement learning, amortized inference, discrete latent variable models, graphical models, and Bayesian optimization (see \cite{gangwani2019selfimitate,han2020discrete,liu2017svp,feng2017amortized,liu2018control,wang2018svmp,gong2019quantile})\\

\noindent Despite the method's success and its effectiveness across a wide range of applications, it often struggles when the target distribution exhibits strongly hidden or isolated modes. To address this limitation, we propose a modification that incorporates a random branching mechanism to enhance the algorithm's exploratory capabilities. \\

\noindent One way to describe particle methods in variational inference is to imagine a collection of projectiles moving through space according to an optimization rule, with their collective behavior forming an empirical measure. Building on this metaphor, our approach replaces these "plain projectiles" with "fireworks": particles that still follow an optimization rule in a piecewise deterministic manner, but now randomly generate new descendants at carefully chosen times, scattering their positions around their parent. \\

\noindent This mechanism allows the algorithm to explore the space more effectively and uncover hidden modes. The introduction of this controlled randomness draws inspiration from branching particle systems. As discussed in detail later, the branched SVGD (BSVGD)  method we present exhibits a systematic upgrade in the approximation of the SVGD, while operating within the same computational time.\\

\noindent We test the effectiveness of our approach using  a numerical approximation of the Wasserstein distance. While this paper focuses on practical implementation and numerical results, the analysis of convergence rates remains an important open challenge that we plan to address in future work.\\

\noindent Before delving further into the details of the method, we break down what we conceive as the most fundamental building blocks. The aim is to emphasize the role of each component rather than its definition, which we hope will make it easier to adapt or improve the method in the future to achieve better performance.\\

\noindent In simple terms, the algorithm combines two key mechanisms: (i) a deterministic refinement applied to an initial measure, and (ii) a random perturbation of the refined measure, introduced via a branching particle system. Although these two components might appear conceptually opposed, they are integrated in the algorithm through an inductive alternation of steps, with the deterministic refinement consistently playing the asymptotically dominant role. \\

\noindent The rest of the paper is organized as follows. Section \ref{Sec:Prelimiaries} introduces the notation and fundamental concepts used throughout. Section \ref{sec:branchedSVGD} presents our main contribution, the BSVGD algorithm, and establishes its theoretical guarantees. Lastly, Section \ref{sec:numerical} focuses on the numerical implementation and performance evaluation of BSVGD in two case studies: mixtures of Gaussian distributions and mixtures of banana-shaped distributions.

\section{Preliminaries}\label{Sec:Prelimiaries}

\noindent In this section, we introduce the mathematical concepts that will be used throughout the paper, including the Wasserstein space, Wasserstein gradient flows, and Stein variational gradient descent.

\subsection{Notation}
Throughout the paper, we denote by $\mathcal{P}(X)$ the set of probability measures on a measurable space $(X,\sigma_X)$. In the case where $\mathcal{X}$ is a normed space, we denote by $\mathcal{P}_p(\mathcal{X})$ the set of probability measures for which the mapping $x\mapsto|x|^p$ is integrable. Given another measurable space $( Y, \sigma_Y )$ and a measurable function $T: ( X, \sigma_X ) \to ( Y, \sigma_Y )$, the push-forward of a measure $\mu\in\mathcal{P}(X)$ under $T$ is denoted by $T_{\#} \mu$. This measure is the unique element in $\mathcal{P}(Y)$ satisfying  
\begin{align*}
    \int_{Y} f(y) \, T_{\#} \mu(\mathrm{d} y) = \int_{X} f(T(x)) \, \mu(\mathrm{d}x),
\end{align*}
for any measurable function $f: ( Y, \sigma_Y ) \to (\mathbb{R}, \mathcal{B}(\R))$.   We denote by $\mathcal{P}_{\mathrm{AC}}(\mathbb{R}^{d})$ the set of absolutely continuous probability measures on $\mathbb{R}^{d}$. For  $\mu \in \mathcal{P}_{\mathrm{AC}}(\mathbb{R}^{d})$, its density function is denoted by $f_{\mu}$. 

\subsection{The Wasserstein Space and Gradient Flows}
\label{sec:wasserstein-gradient-flow}
It is well known in optimal transport theory that the space \( \mathcal{P}_2(\mathbb{R}^d) \) is metrizable and carries a differentiable manifold-like structure (see \cite{santambrogio_optimal_2015} and \cite{chewi2025statistical}). SVGD builds upon these properties, putting particular emphasis on the notion of a tangent space to $\mathcal{P}_2(\R^{d})$. 

\noindent We begin with the definition of the Wasserstein distance: for any \( \mu, \nu \in \mathcal{P}_2(\mathbb{R}^d) \),
\begin{align*}
    d_W(\mu,\nu)
   &:= \left( \inf_{\pi \in \Pi(\mu,\nu)} \int_{\mathbb{R}^{2d}} |x - y|^2 \, \pi(\mathrm{d}x, \mathrm{d}y) \right)^{1/2},
\end{align*}
where \( \Pi(\mu,\nu) \) denotes the set of transport plans between \( \mu \) and \( \nu \); namely, the set of probability measures on \( \mathbb{R}^{2d} \) with marginals \( \mu \) and \( \nu \), respectively. The mapping $d_{W}$ is known as the \textit{2-Wasserstein distance}, and it can be shown that $\mathcal{W}_2 := (\mathcal{P}_2(\RR^d),d_W)$ is in fact a complete metric space, see \cite[Proposition 7.1.5]{ambrosioGradientFlowsMetric2008}.\\

\noindent Furthermore, it is known that \( \mathcal{W}_2 \) possesses a differentiable manifold-type structure. One can implement a formal differential calculus on the Wasserstein space, known in the literature as {Otto calculus}, which can be used to generalize the notion of {gradient flows} to \( \mathcal{W}_2 \). The ideas from this theory can be used to deduce, in a reasonably simple manner, most of the results presented in this chapter. However, for the sake of brevity, we omit its presentation and refer the reader to Chapters II.15--II.23 of \cite{villani2009optimal} for a full review of the ma. Given an open interval $J$ containing zero, we define the tangent plane at a given measure $\nu \in \mathcal{W}_2$ as
\begin{align*}
    \mathcal{T}_{\nu}\mathcal{W}
    :=
    \overline{\big\{ v_0 \ ;\ \{(v_t,\mu_t)\}_{t\in J}\text{ satisfy \eqref{Eq:Continuity-Equation} with }\mu_0=\nu \big\}}^{L^{2}(\nu)},
\end{align*}
where \eqref{Eq:Continuity-Equation} corresponds to the \textit{continuity equation} 
\begin{align}
    \label{Eq:Continuity-Equation}
    \partial_t f_{\mu_t} + (-\nabla)^{*}(f_{\mu_t} v_t) = 0,
\end{align}
with $(-\nabla)^{*}$ referring to the divergence operator (that is, the adjoint operator of $-\nabla$ with respect to the inner product associated with the Lebesgue measure), and $\{v_t : \mathbb{R}^d \to \mathbb{R}^d \}_{t\in J}$ is the vector field associated with an absolutely continuous curve $\{\mu_t\}_{t\in J}$ satisfying $\mu_0 = \nu$. \\

\noindent Let $\mathcal{T}\mathcal{W}$ be the disjoint union of $\mathcal{T}_{\nu}\mathcal{W}$ with $\nu$ ranging over $\mathcal{W}_2$. We refer to any mapping $F:\mathcal{T}\mathcal{W}\rightarrow \mathcal{T}\mathcal{W}$, satisfying
\begin{align*}
    &F[\nu]\in\mathcal{T}_{\nu}\mathcal{W},
    &
    &\forall \nu \in \mathcal{W}_2,
\end{align*}
as a \textit{vector field over} $\mathcal{T}\mathcal{W}$.  This then allows us to define a differential calculus on $\mathcal{W}_2$: We say the a curve $\{\mu_t\}_{t\in J} \subset \mathcal{P}_{\mathrm{AC}}(\R^{d})$, embedded in the Wasserstein space with $J$ an interval around zero, solves the \textit{initial value problem}
\begin{align}
    \label{eq:gradientflowabstract}
    &\partial_t\mu_t = F[\mu_t]
    &\mu_0=\nu,
\end{align}
if the following continuity equation holds
\begin{align*}
    &\partial_t f_{\mu_t} + (-\nabla)^{*}(f_{\mu_t} F[\mu_t]) = 0
    &\mu_0=\nu,
\end{align*}
where $F : \mathcal{TW} \to \mathcal{TW}$ vector field and $\nu \in \mathcal{W}_2$ are given. In this case, we shall also refer to the curve $\{\mu_t\}$ as a \textit{flow on $\mathcal{W}_2$}, or simply as a \textit{flow}  (see \cite[Chapters 8 and 11]{ambrosioGradientFlowsMetric2008}).\\

\subsection{Weak Formulation and Kernelization}

An important detail to emphasize is that the formulation we have presented so far applies only to absolutely continuous curves, which is incompatible with the finite particle methods introduced earlier. To address this limitation, we adopt a weaker notion of gradient flow, defined via the \textit{action of the underlying measures on test functions}: We say that the curve \( \{\mu_t\}_{t \in J} \) solves the system \eqref{eq:gradientflowabstract} in the weak sense if, for every smooth and compactly supported function \( \varphi : \mathbb{R}^d \to \mathbb{R} \), the following identity holds:
\begin{align}
    \label{eq:gradientflowweak}
    \partial_t\langle \mu_t, \varphi \rangle - \langle \mu_t, \nabla \varphi \cdot F[\mu_t] \rangle = 0,
\end{align}
where \( \langle \cdot, \cdot \rangle \) denotes the canonical pairing between measures and test functions. For technical details on this formulation, see \cite[Section 8.3]{ambrosioGradientFlowsMetric2008}. This weak form is particularly well suited to the kernelization approach introduced next, which leads to the formulation of the Stein variational gradient flow. \\

\noindent Let \( V : \mathbb{R}^d \to \mathbb{R} \) be a smooth, symmetric function integrating to one, centered around the target distribution \( \pi \). Define the kernel \( K : \mathbb{R}^d \times \mathbb{R}^d \to \mathbb{R} \) by
\begin{align*}
    K(x,y) := V(x - y), \qquad \forall x, y \in \mathbb{R}^d,
\end{align*}
which in turn induces a family of kernel operators \( \{K_\nu\}_{\nu \in \mathcal{W}_2(\mathbb{R}^d)} \) acting on functions via
\begin{align*}
    K_\nu f(x) := \int_{\mathbb{R}^d} K(x,y) f(y) \, \nu(\mathrm{d}y).
\end{align*}

\noindent The kernelized flow \( K_{\mu}F[\mu] \) is the one associated to the equation
\begin{align*}
    \partial_t f_{\mu_t} + (-\nabla)^{*}[f_{\mu_t} \cdot K_{\mu_t}F[\mu_t]] = 0, \qquad \mu_0 = \nu,
\end{align*}
with corresponding weak formulation
\begin{align*}
    \partial_t \langle \mu_t, \varphi \rangle - \langle \mu_t, \nabla \varphi \cdot K_{\mu_t}F[\mu_t] \rangle = 0,
\end{align*}
for every smooth test function \( \varphi \). For background on this kernelization framework and its analytical implications, see \cite[Chapter 5]{chewi2025statistical}.\\

\noindent An advantage of the previous regularization argument can be seen in the case where empirical measures are considered. This perspective will be explored in more detail in the next section, when discussing the case of the flow associated to the Kullback-Liebler divergence minimization.

\subsection{Minimizers, Kullback-Liebler divergence and SVGD}\label{sec:improvement}
As in the classical case, an important class of curves consists of those constructed from an objective function $\mathcal{Y} : \mathcal{P}_2(\R^d) \to \R$ that a given user aims to progressively minimize. Taking inspiration from vector calculus, the natural candidates for this kind of flows would be vector fields acting as some sort of gradient to $\mathcal{Y}$. More precisely, we say a vector field $\bbNabla\mathcal{Y}$ over $\mathcal{W}_2$ is a \textit{Wasserstein gradient} if it satisfies the following ``chain rule inspired'' equation
\begin{align*}
    \frac{ \rd }{\rd t}\mathcal{Y}[\mu_t]\big{|}_{t=0} &=\int_{\R^{d}}\bbNabla\mathcal{Y}[\nu](x)\cdot v_0(x)\vartheta(\rd x),
\end{align*}
for every $\{\mu_t\}_{t\geq 0}$ satisfying the continuity equation \eqref{Eq:Continuity-Equation} with $\mu_0 = \vartheta$
for some vector field $\{ v_t \}_{t\in J}$, see \cite[Section 5]{chewi2025statistical}.\\

\noindent Among the different type of objective functions $\mathcal{Y}$, we are mostly interested in the case where $\mathcal{Y}$ is the \textit{Kullback-Leibler} (KL) divergence:
\begin{align*}
    \mathrm{KL}(\mu\|\nu)
    :=
    \int_{\R^{d}} \log\left( \frac{ f_{\mu}(x) }{ f_{\nu}(x) } \right) \mu(\rd x),
\end{align*}
for $\mu,\nu\in\mathcal{P}_{\mathrm{AC}}(\R^{d})$ (see \cite[Chapter 2]{cover2006elements} for a summary of the properties of the $\mathrm{KL}$ divergence). The groundbreaking work by Jordan,   Kinderlehrer and Otto from \cite{jordan1998variational}, formulated in the notation utilized in this paper, establishes that the flow associated to the Wasserstein gradient of the mapping $\mu\mapsto \mathrm{KL}(\mu \| \nu)$, under suitable assumptions of the initial condition, has a density satisfying the Fokker-Planck equation 
\begin{align}\label{eq:fokkerplanck}
\partial_tf_{\mu_t}
  &=\Delta f_{\mu_t}-(-\nabla)^*(f_{\mu_t}\nabla \log(f_{\nu})),
\end{align}
which in its weak formulation, yields the evolution 
\begin{align}\label{eq:fokkerplanckweak}
\partial_t\langle \mu_t,\varphi\rangle
  &=-\langle  \mu_t,(\nabla \log(f_{\mu_t})-\nabla  \log(f_{\nu}))\cdot \nabla\varphi\rangle .
\end{align}
This gradient flow naturally raises the question of whether it is possible to numerically implement the evolution \( \mu_t \) as an interpolation curve of measures, in a way that provides asymptotic access to the limiting distribution \( \mu_{\infty} = \pi \) through a numerically tractable procedure. Although this is not straightforward within the framework of particle systems, we can make an adjustment that allows us to apply the methodology by means of the previously introduced kernelization procedure. More specifically, for the functional \( \psi[\nu] := \mathrm{KL}(\nu \,\|\, \pi) \), the Wasserstein gradient of \( \psi \) is, as shown in \cite[Examples 5.11 and 5.12]{chewi2025statistical}, given by
\[
\bbNabla \psi[\nu] = \nabla \log(f_\nu) - \nabla \log(f_\pi).
\]
In contrast, the Stein variational gradient flow is defined as the flow associated with the kernelized vector field 
\[
\nu \mapsto K_{\nu}[\nabla \log(f_{\nu})] + K_{\nu}[\nabla V],
\]
which leads to the following gradient flow, already formulated in its weak form:
\begin{align}\label{eq:SVGDflowweak}
\partial_t \langle \mu_t, \varphi \rangle
  &= -\left\langle \mu_t, K_{\mu_t}[\nabla \log(f_{\mu_t}) - \nabla \log(f_{\pi})] \cdot \nabla \varphi \right\rangle.
\end{align}

\noindent This flow admits a version applicable to an initial distribution of the form
\begin{align*}
    \mu_0 &= \frac{1}{\ell} \sum_{j=1}^{\ell} \delta_{x_j(0)},
\end{align*}
for some \( x_1(0), \dots, x_\ell(0) \in \mathbb{R}^d \). This version can be obtained by considering the limit as \( \varepsilon \to 0 \), where the measure \( \mu_t \) is replaced by its mollification \( \mu_t * \gamma_\varepsilon \), with \( \gamma_\varepsilon \) denoting the centered Gaussian kernel with variance \( \varepsilon \).

\noindent Through elementary computations, one can show that the system of differential equations
\begin{align}
    \label{eq:ODeforSVGD}
    \partial_t x_k(t)
    &= 
    K_{\mu_t}[\nabla \log \rho ](x_k) + \frac{1}{\ell} \sum_{j=1}^\ell \nabla V(x_k - x_j),
\end{align}
where $\rho \propto f_\pi$ up to a normalizing constant, is such that the empirical measure
\begin{align*}
    \mu_t &= \frac{1}{\ell} \sum_{j=1}^{\ell} \delta_{x_j(t)},
\end{align*}
solves the system \eqref{eq:SVGDflowweak}. The measure constructed in this way will be referred to, as the Stein variational gradient flow. For a given element $\mathbf{x}\in\R^{d\ell}$, we can take limit as $t$ goes to infinity in the solution to the system \eqref{eq:ODeforSVGD} with the initial condition $(x_1(0), \ldots, x_\ell(0)) = \mathbf{x}$. The value of this vector will be denoted by $\mathcal{S}_\ell(\mathbf{x})$. The SVGD algorithm can be implemented using the pseudocode presented in the Algorithm \ref{alg:SVGD}. \\

\begin{algorithm}
\caption{{SVGD Algorithm \cite{liu2016svgd}}}\label{alg:SVGD}
\begin{algorithmic}[1]
\Require Score function $\nabla\log \rho(x)$ with support in $\mathbb{R}^d$; initial particles $\{x_i^0\}_{i=1}^\ell$; max iterations $M$; step sizes $\epsilon_d$ for $d=1,\dots,M$; differentiable kernel $k$; convergence threshold $\eta$
\Ensure Set of particles $\{x_i\}_{i=1}^\ell$ approximating the target distribution
\State $d \gets 0$
\State $h \gets 2\eta$
\State Initialize $x_i^0 = x_i$ for $i = 1,\dots,\ell$
\While{$d < M$ and $h > \eta$}
    \For{$i = 1$ to $\ell$}
        \State Compute $\hat{\phi}(x_i^d) = \frac{1}{\ell} \sum_{j=1}^\ell \left[ k(x_j^d, x_i^d) \nabla_{x_j^d} \log \rho(x_j^d) + \nabla_{x_j^d} k(x_j^d, x_i^d) \right]$
        \State $x_i^{d+1} \gets x_i^d + \epsilon_d \cdot \hat{\phi}(x_i^d)$
    \EndFor
    \State $h \gets \frac{1}{\ell} \sum_{i=1}^\ell \|x_i^{d+1} - x_i^d\|$
    \State $d \gets d + 1$
\EndWhile
\State \Return $\{x_i^d\}_{i=1}^\ell$
\end{algorithmic}
\end{algorithm}

\noindent As expected, the effectiveness of this procedure depends significantly on the choice of initial condition. This aspect is leveraged in the present manuscript by introducing modifications to the set of particles forming the empirical measure, through a branching mechanism that will be detailed in the next section. The fact that the system of differential equations \eqref{eq:ODeforSVGD} depends solely on the score function $\nabla \log \rho$ suggests a natural procedure for obtaining a proxy for the measure \( \pi \): we can initialize the system \eqref{eq:ODeforSVGD}, use the score function to evolve it toward its asymptotic limit, and then adopt the resulting empirical measure as an approximation of \( \pi \). In this spirit, the empirical distribution associated to a $\mathbf{x}=(x_1,\dots, x_\ell)$ can be regarded to be "improved", if the components are replaced by the vector $\mathcal{S}_\ell(\mathbf{x})$. Owing to this intuition, we will henceforth refer to the mapping 
\begin{align*}
    \mathbf{x} \longmapsto \mathcal{S}_\ell(\mathbf{x})
\end{align*}
as the \textit{improvement operator} associated to the gradient flow \eqref{eq:ODeforSVGD}.

\subsection{Branching Mechanism}\label{sec:branchingmech}
We proceed to introduce a branching mechanism that will interact with the Stein variational gradient flow by modifying the initial conditions. Let \( \mathcal{C} \) denote a set of labels or colors, given by \( \mathcal{C} = \{E, O, S\} \), where \( E \) alludes to the word  explorer, \( O \) to  optimizer, and \( S \) to  spine. The particles of interest will be pairs \( (x, c) \in \mathbb{R}^d \times \mathcal{C} \). This product space will be denoted by \( \mathcal{U} \).  We are interested in a state space consisting of collections of such elements, so we define the state space \( \mathcal{E} \) by those elements in $\bigcup_{\ell \geq 1} \mathcal{U}^\ell$ that have exactly one component colored  ``$S$". The index $\ell$ describing the number of copies of $\mathcal{U}$ to be considered, will be called the level.\\

\noindent 
Consider a triplet of $\mathbb{N}_0$-supported distributions $ q_E, q_O, q_S $ with finite moments of arbitrary large order. Let the initial configuration be a vector $\mathbf{u} = (u_1, \dots, u_\ell) \in \mathcal{U}^\ell $, where each particle is of the form $ u_i = (x_i, c_i) \in \mathbb{R}^d \times \mathcal{C} $, and observe that we can always identify $\mathbf{u}$ with the pair $( \mathbf{x}, \mathbf{c} )$, where $\mathbf{x} = (x_1, \ldots, x_\ell)$ and $\mathbf{c} = ( c_1, \ldots, c_\ell )$. Finally, consider a fixed Markov kernel $\{ P( \mathrm{d} y | x )\ ;\ x\in\mathbb{R}^d \}$ over $\R^{d}$.\\

\noindent In the branching procedure, each particle \( u_i = (x_i, c_i) \) ramifies  independently to the rest of the particles, according to the following rules:
\begin{enumerate}
    
    \item[(i)] 
        Each particle gives birth to a random number of offspring according to their color; i.e., if $c_i = E$ (resp. $c_i = O$ and $c_i = S$), then the number of particles it produces is distributed $q_E$ (resp. $q_O$ and $q_S$).
    
    \item[(ii)] 
        The number of offspring generated by a "spine" is positive. The number of offspring generated an "optimizer" is zero.
        
    \item[(iii)]
        The new particles generated by $x_i$ are colored as "explorer", with their positions determined by $P( \cdot | x_i )$.

    \item[(iv)]
        The old particles remain in their current position and are recolored as "optimizer".

    \item[(v)]
        After all offspring have been generated, one particle is selected uniformly at random from among the "explorers" and "optimizers", and its color is changed to "spine". All other particles retain their color.
\end{enumerate}

\noindent The resulting collection of particles (positions and updated colors) defines the outcome of the Markov transition. This defines a Markov kernel \( Q \) from \( \mathcal{E} \) to itself.

\section{Branched Stein Variational Gradient Descent}\label{sec:branchedSVGD}

\noindent With the preliminaries established, we are now ready to present the main contribution of our paper: the anticipated branched version of the Stein Variational Gradient Descent. 

\subsection{The algorithm}
The construction of the BSVGD is based on defining an appropriate $\mathcal{E}$-valued Markov chain $\{\mathbf{U}_n\}_{n \geq 0}$, whose transitions incorporate both the improvement operator and the branching mechanism described earlier. The operator $\mathcal{S}_\ell$, introduced at the end of Section \ref{sec:improvement}, naturally acts on elements $\mathbf{u} \in \mathcal{U}^\ell$, by taking the pair $(\mathbf{x}, \mathbf{c})$ and producing $(\mathcal{S}_\ell(\mathbf{x}), \mathbf{c})$, which we associate with 
\begin{align*}
    \mathcal{S}_\ell(\mathbf{u}) := \big( ( \mathcal{S}_\ell(\mathbf{x})_1, c_1 ), \ldots, ( \mathcal{S}_\ell(\mathbf{x})_\ell, c_\ell ) \big).
\end{align*}
For notational simplicity, we will omit the dependence on $\ell$ and write this operation as $\mathcal{S}(\mathbf{u})$ throughout the section. We also fix a triple of distributions $q_E, q_O, q_S$ as in Section \ref{sec:branchingmech}, and denote by $Q$ the corresponding Markov kernel on $\mathcal{E}$.\\

\noindent To construct the Markov chain, we begin with an initial element $\mathbf{u}_0 \in \mathcal{E}$ of level $\ell_0$, and set $\mathbf{U}_0 := \mathcal{S}(\mathbf{u}_0)$; then, the transitions of the chain are   governed by the pushforward $\mathcal{S}_{\#} Q$. To be more precise, let $\mathbf{U}_n \in \mathcal{U}^{\ell_n}$ be the Markov chain at the $n$-th step, identified with the pair $(\mathbf{X}_n, \mathbf{c}_n)$, and let $\mu_n$ denote the empirical distribution associated to the vector $\mathbf{X}_n$.  \\

\noindent At each step, given the current state $\mathbf{U}_n$, we draw an independent sample $\mathbf{u}_{n+1}$ from $Q(\mathbf{U}_n, \cdot)$ viewed as a random element in $\mathcal{E}$, and identify it with the pair $( \mathbf{x}_{n+1}, \mathbf{c}_{n+1} )$. We apply the improvement operator and set $\mathbf{U}_{n+1} := \mathcal{S}(\mathbf{u}_{n+1})$; we then compute the new empirical distribution $\mu_{n+1}$ associated to the new vector of positions $\mathbf{X}_{n+1} = \mathcal{S}(\mathbf{x}_{n+1})$. The resulting sequence of measures $\{ \mu_n \}_{ n \geq 1}$ will be referred to as the outcome of the BSVGD.  The BSVGD algorithm can be implemented using the pseudocode presented in the Algorithm \ref{alg:BSVGD}. \\


\begin{algorithm}
\caption{{Branched SVGD}}\label{alg:BSVGD}
\begin{algorithmic}[1]
\Require Score function $\nabla\log \rho(x)$ with support in $\mathbb{R}^d$; initial particles $\{x_i^0\}_{i=1}^{\ell_0}$; initial labels $\{c_{i}^0\}_{i=1}^{\ell_0}$; max iterations $M$; step sizes $\epsilon_d$ for $d=1,\dots,M$; differentiable kernel $k$; convergence function $\eta(\ell)$; maximum number of particles $L$; distributions $q_E$, $q_O$, $q_S$; conditional distribution $P(\cdot | x)$
\Ensure A set of particles $\{x_i\}_{i=1}^\ell$ approximating the target distribution

\State $X \gets \{ x_i^0 \}_{i=1}^{\ell_0}$
\State $C \gets \{ c_i^0 \}_{i=1}^{\ell_0}$
\State $\ell \gets \#X$ \Comment{\# denotes the cardinality}

\While{$\ell \leq L$}
    \State Update $X$ using Algorithm \ref{alg:SVGD} with parameters $\nabla\log \rho(x), X, \epsilon_d, \eta(\ell)$
    
    \For{$i = 1$ to $\ell$}
        \If{$c_{i} = E$}
            \State Sample $\gamma_i \sim q_E$
        \ElsIf{$c_{i} = S$}
            \State Sample $\gamma_i \sim q_S$
        \EndIf
        \State $c_{i} \gets O$

        \If{$\gamma_i > 0$}
            \For{$j = 1$ to $\gamma_i$}
                \State $x_{ j + \ell +\sum_{k=1}^{i-1} \gamma_k} \sim P(\cdot | x_i)$
                \State $c_{ j + \ell + \sum_{k=1}^{i-1} \gamma_k} \gets E$
            \EndFor
        \EndIf
    \EndFor

    \State Sample $k$ uniformly from $\{1, 2, \dots, \ell + \sum_{i=1}^{\ell} \gamma_i\}$
    \State $c_{k} \gets S$
    \State $X \gets \{ x_i \}_{i=1}^{\ell + \sum_{i=1}^{\ell} \gamma_i}$
    \State $C \gets \{ c_i \}_{i=1}^{\ell + \sum_{i=1}^{\ell} \gamma_i}$
    \State $\ell \gets \#X$
\EndWhile

\State \Return $X$

\end{algorithmic}
\end{algorithm}

\noindent Before proceeding, observe that Algorithm \ref{alg:BSVGD} introduces an additional function $\eta$. Heuristically, this function modulates the precision of the SVGD step at line 5, according to the sample size. This will be discussed with further detail in a later section.

\subsection{Convergence results}
This section aims to discuss features of the output of the algorithm that could potentially guarantee convergence of the outcome of the BSVGD $\mu = \{ \mu_n \}_{n \geq 0}$ towards the target distribution $\pi$.\\

\noindent The reader should be warned from the start that the type of theoretical result one might most naturally hope for: a mild and verifiable condition over $V$ and $Q$ that ensures convergence of the algorithm, is far beyond the scope of this work. This is not merely a limitation of our specific framework, but reflects a broader difficulty in the literature: even for the standard, unbranched version of SVGD, establishing convergence under minimal assumptions remains a formidable challenge. Indeed, while there exists a considerable amount of results proving convergence of SVGD (some even offering convergence rates) none of them are available without imposing some form of non-trivial assumption on the initial condition. The reader can easily verify that a condition of this type is really needed, by thinking of a simple degenerate case: if one initializes the plain SVGD algorithm with a large number of particles, but all of them located at the same position, the evolution of the system will emulate that of a single particle, thereby producing a final state that fails to reflect the true diversity of the target distribution.\\

\noindent Several strategies have been proposed to deal with this issue and to still recover meaningful convergence guarantees. All of them, however, rely on some form of structural assumption that is either incompatible with the empirical setting, or difficult to check in practice. Some approaches rely on the assumption that the initial condition is absolutely continuous, others on assuming that the initialization is drawn from a large random sample with adequate convergence distributional features, and others on the assumption that the initial measure belongs to a sequence that converges to an absolutely continuous one.\\

\noindent The first of these options is clearly ruled out in our case, as the entire framework operates with empirical measures. This leaves us with the other two: either consider a random initialization through random sampling, or a sequence of approximating initial conditions. In this work, we adopt the perspective based on the last approach. Once translated and adapted to the BSVGD framework, it leads to the condition for convergence presented in Theorem \ref{thm:main} below. {The second approach seems to be quite attractive as well, and we intend to address a perspective of this nature in future research work.}\\

\noindent In the sequel, if $\chi$ is a locally integrable, non-negative function over $\R^d$, we will denote by $\mathrm{AC}_{\chi}(\R^d)$ the set of elements in $\mathrm{AC}(\R^{d})$ which density bounded by $\chi$. The distance from an element $\nu \in \mathcal{P}_2(\R^{d})$ to the set $\mathrm{AC}_{\chi}(\R^d)$ will be denoted by $d_W( \nu, \mathrm{AC}_{\chi}(\R^d))$. 

\begin{theorem}\label{thm:main}
Let $\mu_{n}$ denote the outcome of the BSVGD described in Section \ref{sec:branchedSVGD}. Suppose that the moments of order two of $\{\mu_n\}_{n\geq 1}$ are uniformly bounded and that there exists $\chi:\R^{d}\rightarrow\R_{+}$ locally integrable, such that   
\begin{align}
    \label{ineq:dcondforthmmainprime}
    d_W(\mu_n, \mathrm{AC}_{\chi}(\R^d) ) \rightarrow 0.
\end{align}
Then the sequence \( \mu_n \) converges weakly to \( \pi \). In particular, we can guarantee convergence of the $\mu_n$ under the condition 
\begin{align}
    \label{ineq:dcondforthmmain}
    d_W(\mu_n, \mathrm{AC}_y(\R^d) ) \rightarrow 0,
\end{align}
where $\mathrm{AC}_y(\R^d)$ denotes the set of elements in $\mathrm{AC}(\R^d)$ with density bounded by some $y\in\R_{+}$.
\end{theorem}

\noindent The intuition behind condition \eqref{ineq:dcondforthmmain} can be motivated by examining the histogram of \( \mu_n \). If the empirical measure displays no atoms for sufficiently large \( n \), Theorem \ref{thm:main} supports the heuristic that convergence is indeed taking place. This observation, and several other numerical considerations will be presented in Section \ref{sec:numerical}.

\begin{proof}[Proof of Theorem \ref{thm:main}]

\noindent The boundedness of the moments of order two of the $\mu_n$'s imply the sequential compactness property, and hence, to prove the result, it suffices to prove that every convergent subsequence $\{\mu_{n_{k}}\}_{k\geq 1}$ of the $\mu_n$'s has a further subsequence that converges to $\pi$. To this end, we use \eqref{ineq:dcondforthmmain}  a sequence of elements $\{ \nu_{k} \}_{k\geq 1}$ in $\mathrm{AC}(\R^{d})$ with density bounded by $\chi$, satisfying $d_W(\mu_{n_k}, \nu_k )\rightarrow0$. The uniform boundedness of the moments of order two of the $\mu_{n}$'s, together with the boundedness of $d_W(\mu_{n_k},\nu_k)$ implies that $\nu_k$ has moments of order two uniformly bounded, implying the existence of a further subsequence $\nu_{k_m}$ convergent in law. Let $\tau$ denote the weak limit of $\nu_{{k_m}}$. Since the $\nu_{{k_m}}$ have density bounded by $\chi$, by means of Portmanteau's lemma, for every compact $K\subset\R^{d}$,
\begin{align*}
\tau[K]
  &\leq \limsup_{n} \nu_{n_{k_m}}(K) \leq \int_{K}\chi(x)dx.
\end{align*}
By the local integrability of $\chi$, we thus conclude that $\tau$ is absolutely continuous. This observation, combined with the fact that $d_W(\mu_{n_k},\nu_k)\rightarrow0$ implies that $\mu_{n_{k_m}}$ converges in law towards the absolutely continuous measure $\tau$.\\

\noindent We now apply Theorem 7 in \cite{gorham2020stochastic}, to get that the empirical measure associated to $\mathcal{S}(\mathbf{X}_{n_k})$ converges weakly to the limit of the Stein variational gradient flow applied to $\tau$. By \cite[Theorem 3.3]{liu2017svgdflow}, this former probability measure is equal to $\pi$. Since all the $\mathbf{X}_{n}$'s are constructed as the asymptotic limit of the system \eqref{eq:ODeforSVGD}, they are invariant under the action of $\mathcal{S}$, and consequently,  $\mathcal{S}(\mathbf{X}_{n_k})=\mathbf{X}_{n_k}$.  By the previous analysis, the empirical distribution associated to $\mathcal{S}(\mathbf{X}_{n_{k_m}})$ converges to $\pi$, so the identity $\mathcal{S}(\mathbf{X}_{n_k})=\mathbf{X}_{n_k}$ implies that  $\mu_{n_{k_m}}$ converges weakly to $\pi$. We have hence proved that an arbitrary subsequence of $\mu$ has a further subsequence converging to $\pi$, as required.

\end{proof}

\section{Numerical Experiments}
\label{sec:numerical}

\noindent To highlight the suitability of the BSVGD in multimodal cases, as well as its efficiency compared with the classical SVGD, this section focuses on numerical examples. All the codes used to generate the figures presented in this section are public available in the repository isaiasmanuel/SVGD in \href{https://github.com/isaiasmanuel/SVGD}{\textcolor{cyan}{Github}} and were executed in a 24" 2021 iMac with M1 processor.

\subsection{Case studies: Gaussian and Banana-shaped mixtures}
Our first example consists in the mixture of 25 Gaussian densities in $\mathbb{R}^2$, each with a variance of $5\mathbf{I}$, where $\mathbf{I}$ represents the identity matrix. These distributions are arranged in such a way that each of the 25 elements of the Cartesian product $\{0,2,4,6,8 \} \times \{0,2,4,6,8\}$ corresponds to the mean of a different Gaussian, and the weighting parameters of the mixture are given by $\{\frac{1}{325}1,\frac{1}{325}2,...,\frac{1}{325}25\}$, assigned in lexicographical order; e.g. the Gaussian with mean $(0,0)$ has weight $\frac{1}{325}$, the one with mean $(0,2)$ has weight $\frac{2}{325}$, and so on. Visually, the density corresponds to the one shown in Figure \ref{fig:GaussDens}, along with vectorial field of the corresponding score function. This mixture has already been used in literature as a way to test multimodal distribution sampling algorithms, see \cite{wu2023accelerating}. \\

\noindent Our second example follows the idea presented in \cite{benard2023kernel} of using banana-shaped distributions with $t$-tails: initially, the authors discussed that the Stein thining algorithm presented in \cite{liu_kernelized_2016} exhibits spourious minimums for the mixture of banana shaped distributions; to correct this problem, they proposed a variation using a Laplacian correction. 
In the context of multimodal distribution sampling, mixtures of banana-shaped distribution with $t$-tails have been used to exhibit the performance of algorithms, see for example \cite{pompe2020framework}. This is due to the fact that this density is more challenging than the classic Banana shaped Gaussian discussed in \cite{haario1999adaptive}. \\

\noindent Formally, the banana-shaped distribution is defined as follows: let $(x_1,x_2,...,x_d)$ be distributed as a $d$-dimensional $t$-distribution with parameters of location $\mathbf{y}$, scale matrix $\mathbf{\Sigma}=\text{diag}(100,1,...,1)$ and $\mathbf{r}$ degrees of freedom; i.e.
\begin{align*}
    f(\mathbf{x})
    = 
    {\displaystyle { \frac {\Gamma \left[(\mathbf{r} + p)/2\right]}{\Gamma (\mathbf{r}/2) \mathbf{r}^{p/2} \pi^{p/2} \left| \boldsymbol {\Sigma} \right|^{1/2}} }
        \left[ 1+ \frac{1}{\mathbf{r}} ( \mathbf{x} - \mathbf{y} )^{T}{\boldsymbol {\Sigma }}^{-1}( \mathbf {x} - \mathbf{y}) \right]^{ -(\mathbf{r} + p)/2}
    }.
\end{align*}
Then, the \textit{banana-shaped distribution with $t$-tails} is the distribution associated to the vector defined by 
\begin{align*}
    \phi(x_1,x_2, x_3,...,x_d)= (x_1,x_2+bx_1^2-100b,x_3,...,x_d),
\end{align*}
where $b>0$ is a given parameter of nonlinearity. \\

\noindent In our example, we use a mixture of 3 banana-shaped random variables, with locations $(0,0),(0,5),(15,15)$, $b$ parameters $0.03,0.05,0.03$, and weights $0.4,0.4,0.2$, respectively. The density $\rho$ defined by this example and the vectorial field corresponding to the score fucntion $\nabla \log \rho$ are presented in Figure \ref{fig:BananaDens}. \\

\begin{figure}[ht]
    \begin{subfigure}[b]{0.4\textwidth}
    \centering
    \includegraphics[width=1\textwidth]{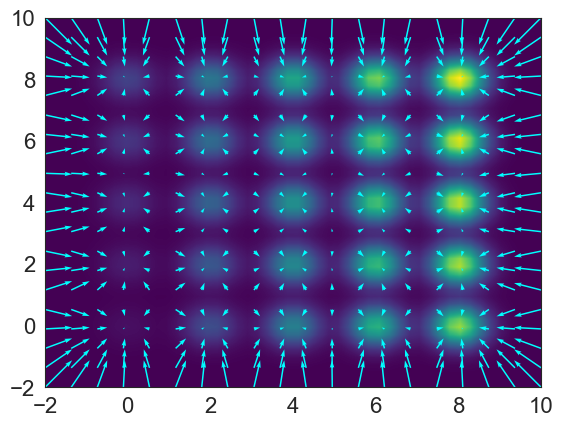}
    \caption{}
    \label{fig:GaussDens}
    \end{subfigure}
    \begin{subfigure}[b]{0.4\textwidth} 
    \centering
    \includegraphics[width=1\textwidth]{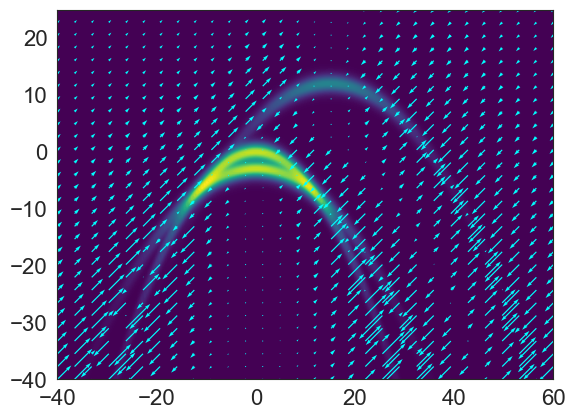}
    \caption{}
    \label{fig:BananaDens}
    \end{subfigure}    
    \caption{$\rho$ densities of interest and the vectorial fields defined by $\nabla \log \rho$. (a) Mixture of Gaussian random variables, (b) Mixture of banana-shaped with $t-$tails random variables.}
\end{figure}

\subsection{Measuring performance}
One of the advantages of working with mixtures of Gaussians and banana-shaped distributions is that we can easily simulate from them. We leverage this property to compare the performance of BSVGD against SVGD. To this end, we use the Wasserstein distance $d_W$ distance, presented in Section \ref{sec:wasserstein-gradient-flow}, to compare two empirical distributions (see also \cite{villani2009optimal}). \\

\noindent Let $\mu$ and $\nu$ be two empirical measures supported on $\{x_1,...,x_\ell\}$ and $\{y_1,...,y_\ell\}$, respectively. The Wasserstein distance between them is given by
\begin{equation}
    \label{eq:Wasserstein}
    {\displaystyle d_W( \mu, \nu )=\inf _{ \sigma }\left( \frac{1}{\ell} \sum _{j=1}^\ell \| x_j - y_{ \sigma(j) } \|^{2}\right)^{\frac {1}{2}},}    
\end{equation}
where the infimum is taken over all the index permutations $\sigma$. In our code this permutation was calculated using the function \textit{linear\_sum\_assignment} of \textit{scipy}.\\ 

\noindent Note that both SVGD and BSVGD generate sequences of empirical measures as the particle positions are updated over time. Our goal is to compare the performance of the two algorithms through these evolving empirical distributions. \\

\noindent We start with the usual SVGD from Algorithm \ref{alg:SVGD}: Let $I_S$ the maximum number of times the vector of positions was updated during the algorithm; then, we define $\{ \mu_{iS} \}_{i = 1}^{ I_S }$ as the sequence such that $\mu_{iS}$ is the empirical measure of the position after the $i$-th update. Now, using the Wasserstein distance as in equation \eqref{eq:Wasserstein}, we can compare each $\mu_{iS}$ with another empirical measure $\pi_{iS}$ of the sample size, defined by sampling independently from the objective $\pi$:
\begin{align}
    \label{eq:pi_is}
    \pi_{iS}:=\frac{1}{\ell} \sum_{j=1}^\ell \delta_{y_{jS}},~y_{jS} \sim \pi,~\forall j = 1, \ldots,~\ell, \forall i = 1, \ldots, I_S.
\end{align}
However, since each $d_W(\mu_{iS}, \pi_{iS})$ is only a point estimator of the real distance, we improve the precision by considering a collection of sequences $\{ \pi_{\cdot S}^a \}_{a = 1}^{ A }$, such that each $\pi_{\cdot S}^a = \{ \pi_{iS}^a \}_{i = 1}^{ I_S }$ is itself an independent copy of $\{ \pi_{jS} \}_{j = 1}^{ I_S }$. Consequently, we define our estimator as the average with respect to this collection of sequences:
\begin{align}
    \label{eq:W_s}
    &W_S(i) := \frac{1}{A} \sum_{a=1}^A d_W(\mu_{iS}, \pi^a_{iS}),
    &
    &\forall i = 1, \ldots, I_S.
\end{align}

\noindent The comparison with the BSVGD follows the same spirit, albeit with a small increase in notational complexity due to the fact that, by construction, the outputs of the BSVGD have an increasing (piece-wise constant) sample size: For each $\ell = 1, \ldots, L$, let $I_B^\ell$ be the maximum number of times the vector of positions was updated during the algorithm at the $\ell$-th level; then, by considering a lexicographic ordering, we define 
\begin{align*}
    \{ \mu_{ j B};~ j = 1, \ldots, J_B \}
    =
    \{ \mu_{i\ell B};~ i = 1, \ldots, I_B^\ell,~\ell = 1, \ldots, L \}
\end{align*}
such that $\mu_{i \ell B}$ is the empirical measure of the position after the $j$-th update at the $\ell$-th level. Similalrly, we have that
\begin{align*}
    &\{ \pi_{ j B};~ j = 1, \ldots, J_B \}
    =
    \{ \pi_{i\ell B};~i = 1, \ldots, I_B^\ell,~\ell = 1, \ldots, L \},
\end{align*}
where each $\pi_{i \ell B}$ is defined as in \eqref{eq:pi_is} with their corresponding level $\ell$, and that $\{ \pi_{\cdot B}^a \}_{a = 1}^{ A }$ is a collection of independent copies of $\pi_{\cdot B} = \{ \pi_{jB} \}_{j = 1}^{ J_B }$. The sequence of estimators $W_B(j)$, $j = 1, \ldots, J_B$ is defined analogusly to \eqref{eq:W_s}.

\subsection{Implementation and results}
For our examples, we run Algorithm \ref{alg:BSVGD} using $\eta(\ell)=\frac{1}{\ell}$ in order to avoid early stops when the sample size increase, i.e. we are being more restrictive in the convergence criterium when the number of particles grow. The flow to the ordinary differential equation \eqref{eq:ODeforSVGD} is approximated by means of an Euler scheme where the step size is set as
\begin{align*}
    \epsilon_d=e_M-\frac{e_M-e_m}{1+e^{-0.01*(d-M*(1/2))}},
\end{align*}
where $e_M$ and $e_m$ are the starting and ending step sizes: 1 and 0.01 in the mixture of Gaussians example, and 10 and 1 in the banana case. By choosing these parameters, we allow big moves at the beginning of the SVGD step, with each successive iteration producing finer movements; this is particularly useful when the offspring is far from the regions with high density, since the point can go fast to the high density region. Other options of $\epsilon_d$ can improve the computational time, e.g. in \cite{liu2016svgd} it is proposed the use of the AdaGrad algorithm introduced in \cite{duchi2011adaptive}. We omit a deeper discussion about this hyperparameter tuning in our comparision due both algorithms using the same function. \\

\noindent To define the position of the offspring in the line 14 of Algorithm \ref{alg:BSVGD}, we use a bivariate Gaussian distribution with mean $x_i$ (their parent), and standard deviation 2 and 5 for the first and second example, respectively. This is the first proposal of how to locate the offspring; nevertheless, adaptative proposal must be explored, as well as the use of mixtures distribution to have local and far descending that allows explore in better ways according to the random variable of interest the space. \\

\noindent The starting points in both examples were taken from a bivariate Gaussian distribution with mean 0 and variance 1. In the SVGD case the sample size is $\ell=500$, and in the BSVGD we take $\ell_0=1$ and $c_0=\{S\}$, that is, we start with only one particle ensured to have offspring. In both algorithms we used a Gaussian kernel defined by 
\begin{align*}
    K_r(x,y):=\pi^{-d/2}e^{- \frac{(x-y)^T(x-y)}{r}},
\end{align*}
with $r=1$, and set the parameter $A = 10$ for the performance estimators $W_S$ and $W_B$.\\

\noindent Regarding the branching mechanism, we set $q_O=0$ with probability one (by definition), $q_S$ a uniform distribution over $\{1,2,3\}$, and 
\begin{align*}
    q_E(x)
    =
    \begin{cases}
        0.5     & \text{if }x=0,\\
        0.2     & \text{if }x=1,\\
        0.3     & \text{if }x=2.\\
    \end{cases}
\end{align*}
The reason behind this configuration is to have a subcritic branching process that allows the sample size to increase slowly. \\

\begin{figure}[!ht]
\centering
    \begin{subfigure}[b]{0.39\textwidth}
    \centering
    \includegraphics[width=1\textwidth]{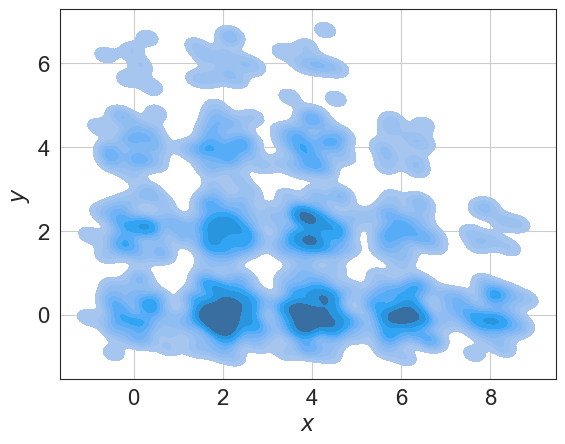}
    \caption{\label{fig:KDEGaussSVGD}}
    \end{subfigure}
    \begin{subfigure}[b]{0.4\textwidth}
    \centering
    \includegraphics[width=1\textwidth]{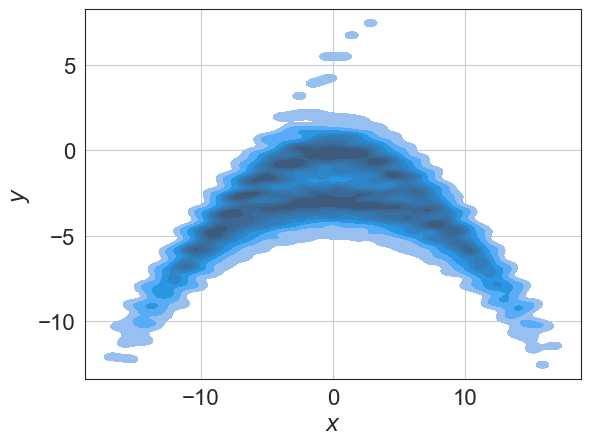}
    \caption{\label{fig:KDEBananaSVGD}}
    \end{subfigure}
    \begin{subfigure}[b]{0.39\textwidth}
    \centering
    \includegraphics[width=1\textwidth]{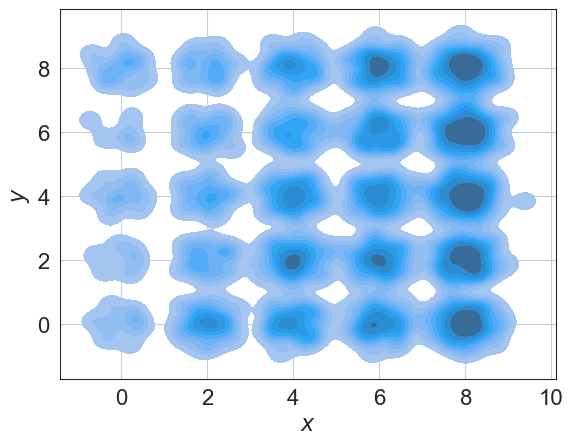}
    \caption{\label{fig:KDEGaussBSVGD}}
    \end{subfigure}
    \begin{subfigure}[b]{0.4\textwidth}
    \centering
    \includegraphics[width=1\textwidth]{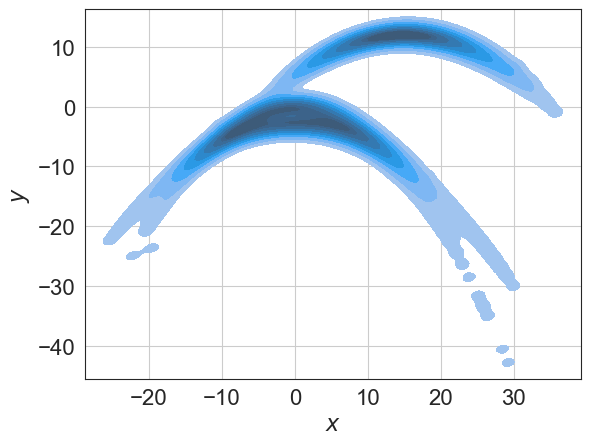}
    \caption{\label{fig:KDEBananaBSVGD}}
    \end{subfigure}

    \caption{Kernel density estimator using the points obtained from SVGD at top and using BSVGD at the bottom, the mixture of Gaussians at left, the mixture of banana shaped distributions at right. }
\label{fig:KDE}

\end{figure}

\noindent In Figure \ref{fig:KDE} we present the kernel density estimators for the points obtained using the SVGD and BSVGD. It is worth noting that even when the sample obtained by BSVGD exhibits an important improvement with respect to the one obtained using SVGD, the SVGD itself has also shown capabilities for detecting multimodality, as it is discussed in \cite{liu2016svgd}. Additionally, the sampling problems that may arise by an early stop on the SVGD could also be solved with more iterations or smallest $\epsilon$ for the stop criterium. Moreover: the BSVGD is computationally more time consuming mainly because the use of SVGD repeatedly; then, in order to compare properly both algorithms it is necessary  not only to see the final samples between both, but to also analyze the performance of each one along the time. \\

\begin{figure}[!ht]
\centering
    \begin{subfigure}[b]{0.39\textwidth}
    \centering
    \includegraphics[width=1\textwidth]{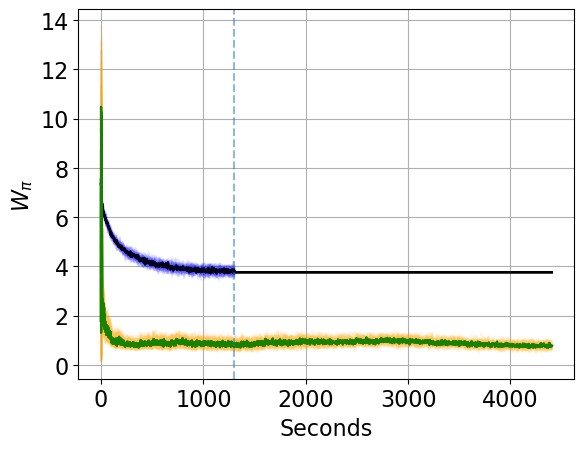}
    \caption{\label{fig:GaussTime}}
    \end{subfigure}
    \begin{subfigure}[b]{0.4\textwidth}
    \centering
    \includegraphics[width=1\textwidth]{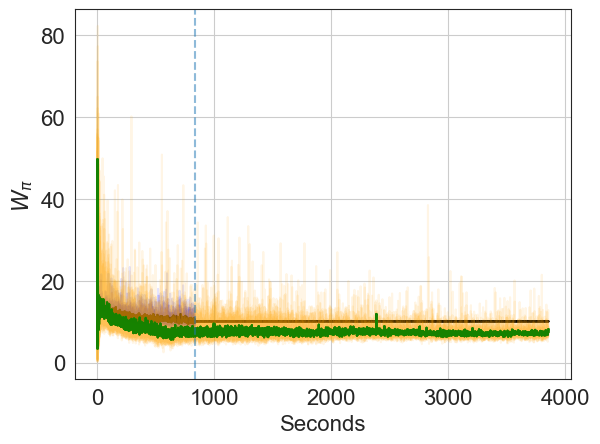}
    \caption{\label{fig:BananaTime}}
    \end{subfigure}

    \begin{subfigure}[b]{0.39\textwidth}
    \centering
    \includegraphics[width=1\textwidth]{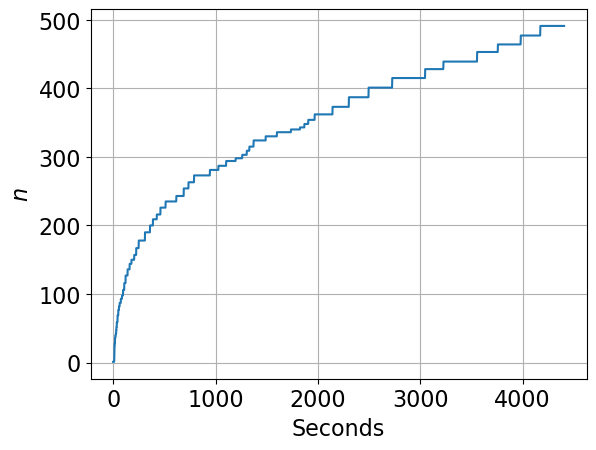}
    \caption{\label{fig:SampleSizeGauss}}
    \end{subfigure}
        \begin{subfigure}[b]{0.4\textwidth}
    \centering
    \includegraphics[width=1\textwidth]{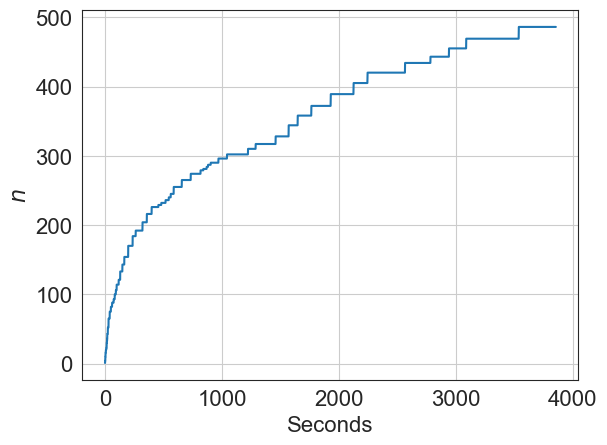}
    \caption{\label{fig:SampleSizeBanana}}
    \end{subfigure}

    \caption{At left the mixture of Gaussian, at right the banana shaped case. At top in black the function $(t_i,W_S(i))$ is presented, where $t_i$ is the time spent to obtain the $i$-th update using the SVGD algorithm and in blue the $A$ functions used in the average to obtain $W_S$ presented until the algorithm convergence. The vertical dashed line is the time when algorithm \ref{alg:SVGD} converges, in green are the $(t_j,W_B(j))$, and in orange are the functions used in its average. At bottom the sample size of the BSVGD at time $t$}
\label{fig:Times}

\end{figure}

\noindent In Figure \ref{fig:Times} we present $(t_i,W_S(i))$ where  $t_i$ is the time required to calculate the $i$-th particles' update, and analogous for $(t_j,W_B(j))$; we also present the sample size of the BSVGD along time. Observe that the BSVGD is computationally more time consuming than the classical SVGD. Nevertheless, we want to remark that if we let the BSVGD run for the same amount of time that takes the SVGD to converge, in our examples the graphs of the function $W_B$ fall under $W_S$. Therefore, an early stopped BSVGD seems to be a good option in contrast to executing the SVGD when the computational time is limited, with the caveat that the sample size will be lower. \\

\noindent Based on the previous examples, we can affirm that the BSVGD is an effective algorithm in multimodal cases with respect to the classical SVGD when $\rho$ presents multimodality. This is due to our two-fold algorithm: the SVGD step accommodates the points, first towards the mode and then towards the tails, while the branching step encourages the exploration of the particles after these have been arranged by the SVGD, preparing them for the next cycle. \\

\subsection{Conclusions and Further Work}

\noindent The BSVGD emerges as a competitive algorithm in different directions. In practical problems, we can obtain a sample that reflects better the multimodality compared with the classical SVGD. \\

\noindent It is important to remark that when the modes of $\rho$ have big valleys between them, the BSVGD struggles to detect the mixture weights properly. Because of this and with the aim to improve the sample between modes, it is necessary to explore new candidates for the branching and exploring distributions. A natural candidate for this include adaptive proposals. \\

\noindent It will be important in future works to also modify the selection of the spine. Instead of taking it uniform between the points, we could take weighting of the points based on $\rho$, if we can evaluate it. This idea is aligned with the work presented in \cite{pompe2020framework}, with the difference that we are searching the modes while also generating an approximated sample.\\

\noindent \textbf{Acknowledgements}\\
Isa\'ias Ba\~nales was supported by JSPS KAKENHI (Grant 21H05200 and 24K17144) and JST SATREPS project (Grant JPMJSA2310).\\
Arturo Jaramillo Gil was supported by
the grant CBF2023-2024-2088.
\\
Joshu\'e Hel\'i Ricalde-Guerrero gratefully acknowledges the support of the SNF project MINT 205121-21981.



\bibliographystyle{plain}
\bibliography{references.bib}

\end{document}